\documentclass[12pt,hidelinks]{article} %***
\usepackage[sectionbib,sort]{natbib}
\usepackage{array,epsfig,fancyheadings,rotating}
\usepackage[]{hyperref}  %<----modified by Ivan
\usepackage{sectsty, secdot}
\sectionfont{\fontsize{12}{14pt plus.8pt minus .6pt}\selectfont}
\renewcommand{\theequation}{\thesection\arabic{equation}}
\subsectionfont{\fontsize{12}{14pt plus.8pt minus .6pt}\selectfont}
\usepackage{amsmath}
\usepackage{amssymb}
\usepackage{amsfonts}
\usepackage{multirow}
\usepackage{amsthm}
\usepackage{bm, enumerate,xcolor}
\usepackage{graphicx,multicol}
\usepackage{float,subcaption}
\usepackage{multirow}
\usepackage{epstopdf}
\usepackage[plain,noend]{algorithm2e}

\graphicspath{{sim_plot/}}

\usepackage{xr}
\allowdisplaybreaks

\newtheorem{theorem}{Theorem}
\newtheorem{lemma}{Lemma}

\newtheorem{proposition}{Proposition}
\theoremstyle{definition}
\newtheorem{definition}{Definition}

\newtheorem{remark}{Remark}
\newcommand{\bfbeta}{\mbox{\boldmath $\beta$}}

\newcommand{\bfdelta}{\mbox{\boldmath $\delta$}}

\begin{document}

%%%%%%%%%%%%%%%%%%%%%%%%%%%%%%%%%%%%%%%%%%%%%%%%%%%%%%%%%%%%%%%%%%%%%%%%%%%%%%%%%%%%%%%%%%%%%%%%%%%%%%%%%%%%%%%%%%%%%%%%%%%%
%%%%%%%%%%%%%%%%%%%%%%%%%%%%%%%%%%%%%%%%%%%%%%%%%%%%%%%%%%%%%%%%%%%%%%%%%%%%%%%%%%%%%%%%%%%%%%%%%%%%%%%%%%%%%%%%%%%%%%%%%%%%

\renewcommand{\baselinestretch}{1}

\markright{ \hbox{\footnotesize\rm Statistica Sinica
%{\footnotesize\bf 24} (201?), 000-000
}\hfill\\[-13pt]
\hbox{\footnotesize\rm
%\href{http://dx.doi.org/10.5705/ss.20??.???}{doi:http://dx.doi.org/10.5705/ss.20??.???}
}\hfill }

\markboth{\hfill{\footnotesize\rm FIRSTNAME1 LASTNAME1 AND FIRSTNAME2 LASTNAME2} \hfill}
{\hfill {\footnotesize\rm FILL IN A SHORT RUNNING TITLE} \hfill}

\renewcommand{\thefootnote}{}
$\ $\par

%%%%%%%%%%%%%%%%%%%%%%%%%%%%%%%%%%%%%%%%%%%%%%%%%%%%%%%%%%%%%%%%%%%%%%%%%%%%%%%%%%%%%%%%%%%%%%%%%%%%%%%%%%%%%%%%%%%%%%%%%%%%

\fontsize{12}{14pt plus.8pt minus .6pt}\selectfont \vspace{0.8pc}
\centerline{\large\bf Elastic-net Regularized High-dimensional }
\vspace{2pt} \centerline{\large\bf Negative Binomial Regression:}
\vspace{2pt} \centerline{\large\bf Consistency and Weak Signal Detection}
\vspace{.3cm}
\centerline{Huiming Zhang$^{1,3,4}$, Jinzhu Jia$^{2,3}$} % \vspace{.3cm}
\vspace{.3cm}\centerline{\it School of Mathematical Sciences$^{1}$, School of Public Health$^{2}$ }
\vspace{.3cm}\centerline{\it and Center for Statistical Science$^{3}$, Peking University, Beijing, P.R. China;}
\vspace{.3cm}\centerline{\it Department of Mathematics$^{4}$, Faculty of Science and Technology,}
\vspace{.3cm}\centerline{\it University of Macau, Taipa, Macau, P.R. China.}
\vspace{.55cm} \fontsize{9}{11.5pt plus.8pt minus
.6pt}\selectfont

%%%%%%%%%%%%%%%%%%%%%%%%%%%%%%%%%%%%%%%%%%%%%%%%%%%%%%%%%%%%%%%%%%%%%%%%%%%%%%%%%%%%%%%%%%%%%%%%%%%%%%%%%%%%%%%%%%%%%%%%%%%%

\begin{quotation}
\noindent {\it Abstract:}
We study a sparse negative binomial regression (NBR) for count data by showing the non-asymptotic advantages of using the elastic-net estimator. Two types of oracle inequalities are derived for the NBR's elastic-net estimates by using the Compatibility Factor Condition and the Stabil Condition. The second type of oracle inequality is for the random design and can be extended to many $\ell_1 + \ell_2$ regularized M-estimations, with the corresponding empirical process having stochastic Lipschitz properties. We derive the concentration inequality for the suprema empirical processes for the weighted sum of negative binomial variables to show some high--probability events. We apply the method by showing the sign consistency, provided that the nonzero components in the true sparse vector are larger than a proper choice of the weakest signal detection threshold. In the second application, we show the grouping effect inequality with high probability. Third, under some assumptions for a design matrix, we can recover the true variable set with a high probability if the weakest signal detection threshold is large than the turning parameter up to a known constant. Lastly, we briefly discuss the de-biased elastic-net estimator, and numerical studies are given to support the proposal.

\vspace{9pt}
\par
\end{quotation}\par

\def\thefigure{\arabic{figure}}
\def\thetable{\arabic{table}}

\renewcommand{\theequation}{\thesection.\arabic{equation}}

\fontsize{12}{14pt plus.8pt minus .6pt}\selectfont

\setcounter{section}{0} %***
\setcounter{equation}{0} %-1

\lhead[\footnotesize\thepage\fancyplain{}\leftmark]{}\rhead[]{\fancyplain{}\rightmark\hspace{1cm} \footnotesize\thepage}%Put this line in Page 2

\section{Introduction}
In this study, we focus on regression problems involving count data (sometimes called categorical data). The responses are denoted as $\{ {Y_i}\} _{i = 1}^n$, each of which follows a univariate discrete distribution. Here, the covariates $\{{\emph{\textbf{X}}_i} := {({x_{i1}}, \cdots ,{x_{ip}})^T}\}_{i = 1}^n \in \mathbb{R}^P$ are supposed to be a deterministic or random variable. If they are random, we can deal with the model by conditioning on design matrix $\textbf{X}:=({\emph{\textbf{X}}_1}, \cdots ,{\emph{\textbf{X}}_n})^T$. The conditional expectation of $Y_i|{\emph{\textbf{X}}_i^T}$ is related to ${\emph{\textbf{X}}_i^T}\bfbeta^{*}$ after a transformation using a link function, where $ \bfbeta^*  = {({\beta _1^*}, \cdots ,{\beta _p^*})^T}$ is the unknown true coefficient vector. The Poisson regression is a well-known example.  Covariates in a count data regression may take discrete or continuous values. Here, important examples includes logistic regression, Poisson regression and negative binomial regression (NBR), amony others. There are many monographs on statistical models for counting data; see for example, \cite{Hilbe11} and \cite{Tutz11}.

A commonly used regression model for count data is the Poisson generalized linear model, particularly in the economic, social, and biological sciences, see \cite{Tutz11}. A Poisson regression considers that the response variables $Y_i$'s are nonnegative integers that follow the Poisson distribution, i.e. $P({Y_i=y_i}{\left| \lambda  \right._i}) = \frac{{\lambda _i^{{y_i}}}}{{{y_i}!}}{e^{ - {\lambda _i}}}$ for $i = 1,2, \cdots ,n$, where the expectation of ${Y_i}$ is ${\lambda _i} := {\rm{E}}{(Y_i)}$. We require that the positive parameter ${\lambda _i}$ be related to a linear combination of $p$ covariate variables. Specifically, the Poisson regression assumes the logarithmic link function $\eta ({\lambda _i}) = :\log \lambda _i = {\emph{\textbf{X}}_i^T}\bfbeta^*$. Owing to the nature of the Poisson distribution, the variance is equal to the expectation: ${\rm{E(}}{Y_i}\left| {{\textit{\textbf{X}}_i}} \right.) = {\rm{Var(}}{Y_i}\left| {{\textit{\textbf{X}}_i}} \right.) = {\lambda _i}$, called \textit{equidispersion}.

However, in practice, we often encounter \textit{overdispersion}. In this case, the variance of count data is greater than the mean comparing to Poisson count data. For example in RNA-Seq gene expression data, the negative binomial (NB) distribution provides a good choice for modeling a set of count variables and related high-dimensional sets of quantitative or binary variables are of interest, that is $p \gg n$. As evidence of overdispersion, in real data, the variance of the response variable is greater
than its mean; see \cite{Rauschenberger16} and \cite{Qiu18}.  To test whether the variance of count data is greater than the expectation, \cite{Cameron1990} proposed the Cameron--Trivedi test:
\begin{align*}
{{\rm{H}}_0}{\rm{:Var(}}{Y_i}\left| \textit{\textbf{X}}_i \right.) = {\rm{E(}}{Y_i}\left| \textit{\textbf{X}}_i \right.){\rm{ = :}}{\mu _i}~~{\rm{vs.}}~~{{\rm{H}}_1}{\rm{:Var(}}{Y_i}\left| \textit{\textbf{X}}_i \right.) = {\mu _i}{\rm{ + }}\alpha g({\mu _i}),
\end{align*}
where $g({\mu _i}) = {\mu _i}$ or $g({\mu _i}) = \mu _i^2$, and the constant $\alpha $ is the value to be tested. Therefore, the hypothesis test is alternatively written as ${{\rm{H}}_0}{\rm{:~}}\alpha  = 0~~{\rm{vs.}}~~{{\rm{H}}_1}{\rm{:~}}\alpha  \ne 0$. For $\alpha  \ne 0$, the count data is overdispersed if $\alpha  > 0$, and it is underdispersed if $\alpha  < 0$. Here the \emph{underdispersion} means that the variance of the data is less than the mean, which suggests that a binomial regression (see Section 3.3.2 of \cite{Tutz11}) or a COM-Poisson regression (see \cite{Sellers2008}) should be suitable. More details on the overdispersion test can be found in Chapter 7 of \cite{Hilbe11}.

When testing for overdispersion, we have to correct the hypothetical distributions and select a flexible distribution, such as some two-parameter models. A suggested overdispersed distribution is the negative binomial (NB) distribution that is a particular case of the discrete compound Poisson (DCP) family. NB also belongs to the class of infinitely divisible distribution. For more detailed NB and DCP distributions properties, please refer to Section 5.9.3 of \cite{johnson05} and \cite{Zhang14}.

In low-- and fixed--dimensional regressions with $p<n$, researcher often use the maximum likelihood estimator (MLE) of the regression coefficients. Here, we employ the \emph{average negative log-likelihood function} of the NBR (i.e. a convex empirical process indexed by $n$):
\begin{equation*}
\ell_n(\bfbeta ) := - \frac{1}{n}\sum\limits_{i = 1}^n {[{Y_i}{{{\textit{\textbf{X}}_i^T}\boldsymbol{\beta} }}  - (\theta  + {Y_i})\log (\theta  + {e^{{\textit{\textbf{X}}_i^T}\boldsymbol{\beta} }})]},~\boldsymbol{\beta} \in \mathbb{R};
\end{equation*}
see Section~\ref{nbr}. Here, $\ell_n(\bfbeta )$ is also termed the empirical NBR loss function in the field of machine learning point. If $\theta$ is given (or treated as a tuning parameter), the NBR actually belongs to the class of generalized linear models (GLMs) with noncanonical links. It should be noted that the coefficient of ${Y_i}$ in the log-likelihood of a common GLM with a canonical link function is linear in $\textit{\textbf{X}}_i^T\bfbeta$, whereas the coefficient of ${Y_i}$ in the log-likelihood of the NBR is nonlinear in $\textit{\textbf{X}}_i^T\bfbeta$ owing to the noncanonical link function.

In a high-dimensional setting, a powerful tool for remedying the MLE is to add the penalty function to the $\ell_n (\bfbeta )$ to get the penalized (regularized) likelihood estimator. Here, we study the elastic-net regularized MLE defined as follows.
\begin{definition} \label{def-en}
(Elastic-net method of NBR) For the empirical NB loss function $\ell_n (\bfbeta )$, let ${\lambda _1},{\lambda _2}> 0$ be tuning parameters. Then, the elastic-net estimates are defined as
\begin{equation}\label{eq:en}
\boldsymbol{\hat \beta} =: \boldsymbol{\hat \beta} ({\lambda _1},{\lambda _2})=\mathop {\rm{argmin}}\limits_{\boldsymbol{\beta}   \in {\mathbb{R}^p}} \{ \ell_n (\bfbeta ) + {\lambda _1}\left\| \boldsymbol{\beta}   \right\|_1 + {\lambda _2}{\left\| \boldsymbol{\beta}  \right\|_2^2}\},
\end{equation}
where $\|\bfbeta\|_q :={( {\sum_{i = 1}^p {{{\left| {{\beta _i}} \right|}^q}} } )^{1/q}} $ is the $l_q$-norm of $\beta$, for $1 \le q < \infty $.
\end{definition}
In the section below, we usually denote ${\hat \bfbeta }$ as $\hat \bfbeta ({\lambda_1,\lambda_2})$, for simplicity.

Chapter 3 of \cite{Tutz11} begins with three important criteria for penalized estimation methods for sparse coefficient vectors:

{\normalsize{
$ {1^ \circ }$. \emph{Existence of unique estimates: this is where MLEs often fail;}

 ${2^ \circ }$. \emph{Prediction accuracy: a model should yield a decent prediction of the outcome;}

 ${3^ \circ }$. \emph{Sparseness and interpretation: a parsimonious model that contains the strongest effects is easier to interpret than a big model with hardly any structure.}
}}

For ${3^ \circ }$, as the penalty function, we study the elastic-net estimate because it enjoys the advantages of both the Lasso and the Ridge, see \cite{Zou2005}. The Lasso can only select one variable in a group of highly related variables, whereas the elastic-net can choose more than one, which we called a grouping effect. For ${1^ \circ }$ and ${2^ \circ }$, we concentrate on the nonasymptotic oracle inequalities of the elastic-net penalized MLE in NB regression because asymptotic distribution of the high-dimensional penalized estimator is usually not available. Essentially, deriving oracle inequalities is a powerful mathematical skill that
gives deep insight into an estimator's nonasymptotic fluctuation compared to that of an ideal unknown parameter (the oracle). \cite{Wang2016} compared the NBR and Poisson regression models based on the elastic-net, MCP-net, and SCAD-net penalty functions by using hospitalization days in hospitalized pediatric cardiac surgery and the associated covariates for variable selection analysis. \cite{Massaro16} constructed the elastic-net penalized NBR to analyze overdispersed count data: time-to-death (in days). Here, the elastic-net selects the genes' functional characteristics that increase or decrease the survival time in the high-dimensional scenario, as $p \gg n$. In practice, the covariates are usually corrupted because they contain unavoidable measurement errors. \cite{Sorensen2018} suggested that elastic-net penalty (or generalized elastic-net penalty with higher-order terms, such as cubic, quadratic terms, etc.) can decorrupt the corrupted covariates in high-dimensional GLMs, by choosing the second tuning parameter in the elastic-net.\\

\textbf{Contributions}:
\begin{itemize}
\item For GLMs, \cite{Bunea08} investigated the oracle inequalities in the setting of logistic and linear regression models for the elastic-net penalization schemes under the Stabil Condition. By extending the proofs from \cite{Bunea08}, \cite{Blazere14} derived oracle inequalities for GLMs with canonical link functions that do not contain the NBR. The empirical processes technique is used by \cite{Blazere14} to get the oracle inequalities for elastic-net in GLMs; however, their assumption of GLMs does not contain the NBR. Even under a fixed design, the Hessian matrix of the NB log-likelihood contains random responses. This complex phenomenon is substantially different from the canonical link GLMs. Additional treatments for the concentration of a random Hessian matrix are needed. To show the KKT-like event with high probability, we propose a new concentration inequality for the superma of multiplier NB empirical processes.

 \item  \cite{Geer2008} mainly studied the oracle inequalities for high-dimensional GLMs with Lipschitz loss functions. However, the loss of NBR is not Lipschitz owing to the unbounded responses. To handle the non-Lipschitz loss, we have to ensure the stochastic Lipschitz property (see \cite{Chi10}) of the NB loss with high probability. Thus we derive oracle inequalities for elastic-net estimates for the NBR under the Compatibility Factor Condition or Stabil Condition, which differs from the conditions in \cite{Geer2008}.

 \item Apart from the $\ell_{1}$ consistency, few studies focus the sign consistent (\cite{Zhao06}) of the elastic-net type estimators, see \cite{Jia2010} for the linear model, and \cite{Yu2010} for the Cox model. Based on the bounded covariates assumption,  we study the sign consistency of an elastic-net regularized NBR without using the \emph{Irrepresentable Condition} in \cite{Zhao06}.

\end{itemize}

We examine the theoretical properties of the elastic-net methods for a sparse estimator in the NBR within the framework of the nonasymptotic theory. Section~\ref{nbr} and Section~\ref{KKT} present a review of the NBR and KKT conditions. In Section~\ref{compatibility} and \ref{Stabil}, we show that two types of oracle inequalities can be derived for $\ell_{1}$ estimation and prediction error bound under the assumption of the Compatibility Factor Condition or Stabil Condition with measurement errors. The remaining sections are byproducts of our proposed oracle inequalities. We establish a uniform bound for the grouping effect in Section~\ref{sec-GE}. To obtain the sign consistency in Section~\ref{sec-IC0}, we require a uniform signal strength in order to detect coefficients larger than a constant multiplied by the tuning parameter of the $\ell_{1}$ penalty. Using the weakest signal condition, in Section~\ref{signals}, we find that the probability of correct inclusion for all true variables in the selected set and the probability of corrected subset selection are high. We discuss de-biased elastic-net regularized M-estimators for low-dimensional parameters in Section~\ref{de-biased}. All proofs of the main theorems, lemmas, and propositions are given in Appendix S1, and the assisted lemmas are presented in Appendix S2. Simulation studies are provided in Appendix S3.

\section{High-dimensional NBR}
\label{High}
In the following two subsections, we review the negative binomial GLMs and the corresponding mathematical optimization problems.

\subsection{NBR}\label{nbr}

 The probability mass function of the negative binomial distribution random variable  is ${p_y} = :P(Y = y) = \frac{{\Gamma (y + \theta)}}{{\Gamma (\theta)y!}}{(1 - p)^\theta}{p^y},(p \in (0,1),y \in \mathbb{N})$. The expectation and variance of the NB distribution are $\frac{{\theta p}}{{1 - p}}$ and $\frac{{\theta p}}{{{{(1 - p)}^2}}}$, respectively. If $\theta$ is a positive integer, it is called a Pascal distribution. This special case of the NB is modeled as the number of failures $Y = y$ before the $\theta$-th success in repeated mutually independent Bernoulli trials (with success probability $1-p$). Here, $\theta$ is a positive integer or real number.

 In the regression setting, one type of NBR assumes that the count data response obeys the NB distribution (denoted as $Y \sim {\rm{NB}}({\mu _i},\theta)$) with over-dispersion:
\begin{align*}
P({Y_{{i}}} &= {y_i}|\textit{\textbf{X}}_i) = :f({y_i},\theta ,{\mu _i}) = \frac{{\Gamma (\theta  + {y_i})}}{{\Gamma (\theta ){y_i}!}}{(\frac{{{\mu _i}}}{{\theta  + {\mu _i}}})^{{y_i}}}{(\frac{\theta }{{\theta  + {\mu _i}}})^\theta },(i = 1,2, \cdots ,n)
\end{align*}
Here, ${\rm{E(}}{Y_i}\left| \textit{\textbf{X}}_i \right.) = {\mu _i}$ and ${\rm{Var(}}{Y_i}\left| \textit{\textbf{X}}_i \right.) = {\mu _i} + \frac{{\mu _i^2}}{\theta }$. The $\theta $ is a qualification of the level of overdispersion that underlies a count data set. Furthermore, $\theta $ is assumed as the known dispersion parameter which can be estimated (see Section 8 of \cite{Hilbe11}). When the mean parameter ${\mu _i}$ and the covariates are linked by ${\rm{log}}{\mu _i} = {\textit{\textbf{X}}}_i^T\boldsymbol{\beta}^{*}$, we have an NBR. When $\theta  \to  + \infty $, ${\rm{Var(}}{Y_i}\left| \textit{\textbf{X}}_i \right.) \to {\mu _i}{\rm{ = E(}}{Y_i}\left| \textit{\textbf{X}}_i \right.)$. Thus, the Poisson regression is a limiting case of the NBR when the dispersion parameter tends to infinite. Because overdispersion occurs in real data, the NBR can be more powerful and interpretable than a Poisson regression.

The log-likelihood function of the NB responses is:
\begin{align*}
&~~L(\emph{\textbf{Y}};\boldsymbol{\beta} ) = \log [\prod \limits_{i = 1}^n {f({Y_i},\theta ,{\mu _i})} ] = \sum\limits_{i = 1}^n {\log } \{ \frac{{\Gamma (\theta  + {Y_i})}}{{\Gamma (\theta ){Y_i}!}}{(\frac{{{\mu _i}}}{{\theta  + {\mu _i}}})^{{Y_i}}}{(\frac{\theta }{{\theta  + {\mu _i}}})^\theta }\} \\
& = \sum\limits_{i = 1}^n {\{ \log \Gamma (\theta  + {Y_i}) + {Y_i}\log {\mu _i} + \theta \log \theta  - \log \Gamma (\theta )} - \log {Y_i}! - (\theta  + {Y_i})\log (\theta  + {\mu _i})\}  \\
&= c_0 + \sum\limits_{i = 1}^n {[{Y_i}{{{\textit{\textbf{X}}_i^T}\boldsymbol{\beta} }}  - (\theta  + {Y_i})\log (\theta  + {e^{{\textit{\textbf{X}}_i^T}\boldsymbol{\beta} }})]},~\text{with a constant}~c_0.
\end{align*}
Then, take the derivative of the vector $\bfbeta $. Let $\frac{{\partial L(\emph{\textbf{Y}};\boldsymbol{\beta} )}}{{\partial \boldsymbol{\beta} }} := {\{ \frac{{\partial L(\emph{\textbf{Y}};\boldsymbol{\beta} )}}{{\partial {\beta _1}}}, \cdots ,\frac{{\partial L(\emph{\textbf{Y}};\boldsymbol{\beta} )}}{{\partial {\beta _p}}}\} ^T}$. We get the \emph{score function}
\begin{equation}\label{eq:NBEP}
{\dot \ell_n ({\bfbeta})}:= - \frac{1}{n}\frac{{\partial L(\emph{\textbf{Y}};\boldsymbol{\beta} )}}{{\partial \boldsymbol{\beta} }} =  - \frac{1}{n} \sum\limits_{i = 1}^n {{\textit{\textbf{X}}_i}\theta [} \frac{{\theta  + {Y_i}}}{{\theta  + {e^{{\textit{\textbf{X}}_i^T}\boldsymbol{\beta} }}}} - 1] = - \frac{1}{n} \sum\limits_{i = 1}^n {\frac{{{\textit{\textbf{X}}_i}({Y_i} - {e^{{\textit{\textbf{X}}_i^T}\boldsymbol{\beta}}})\theta }}{{\theta  + {e^{{\textit{\textbf{X}}_i^T}\boldsymbol{\beta} }}}}}.
\end{equation}
By setting ${\dot \ell_n ({\bfbeta })}= 0$, we obtain the solution ${\boldsymbol{\hat \beta}_{mle}}$. The second derivative is calculated as the \emph{Hessian matrix}
$
\ddot \ell_n(\bfbeta )=\frac{1}{n} \sum\limits_{i = 1}^n {{\textit{\textbf{X}}_i}{\textit{\textbf{X}}_i^T}} \frac{{\theta (\theta  + {Y_i}){e^{{\textit{\textbf{X}}_i^T}\boldsymbol{\beta} }}}}{{{(\theta  + {e^{{\textit{\textbf{X}}_i^T}\boldsymbol{\beta} }})^2}}}
$, which is semi-negative, such that ${\hat \beta _{mle}}$ makes the likelihood function take the maximum value globally.

\subsection{KKT conditions}
\label{KKT}

Let $g(\bfbeta)$ be a nonnegative convex function with $g(\textbf{0}) = \textbf{0}$, and ${\lambda _1}$ and ${\lambda _2}$ be positive turning parameters. \cite{Yu2010} considered a penalized likelihood for the convex loss function $ \ell (\boldsymbol{\beta} ) $,
$$
F(\boldsymbol{\beta} ;{\lambda _1},{\lambda _2}) = \ell_n(\boldsymbol{\beta} ) + {\lambda _1}{\left\| \boldsymbol{\beta}  \right\|_1} + {\lambda _2}g(\boldsymbol{\beta})
$$
as the generalized Lasso-type convex penalty (GLCP). The GLCP estimator for the general log-likelihood is $\boldsymbol{\hat \beta} ({\lambda _1},{\lambda _2}) = \mathop {\rm{argmin}}_{\boldsymbol{\beta}  \in \mathbb{R}^p} F(\boldsymbol{\beta} ;{\lambda _1},{\lambda _2})$.
By the sub-derivative technique in the optimization function, the corresponding Karush--Kuh--Tucker(KKT) conditions  of GLCP estimator are
\begin{eqnarray}\label{eq:kkt}
\left\{
\begin{aligned}
{\dot \ell_{n,j}}(\hat \bfbeta )+{\lambda _2}{\dot g_j}(\hat \bfbeta) = - {\lambda _1} {\rm{sign}}(\hat\beta_j) \quad\, \text{  if } \hat\beta_j\neq 0,\\
|{\dot \ell_{n,j}}(\hat \bfbeta )+{\lambda _2}{\dot g_j}(\hat \bfbeta)| \le {\lambda _1} \qquad\quad\text{  if } \hat\beta_j=0,
\end{aligned}
\right.
\end{eqnarray}
See page 68 of \cite{Buhlmann11}). Thus, in the NBR, the KKT conditions for the non-zero (or zero) elastic-net estimate is

\begin{lemma}[Necessary and Sufficient Condition]\label{lem:iff}
 Let $k \in \{ 1,2, \cdots ,p\} $ and ${\lambda _2} > 0$.  Then, a necessary and sufficient condition for elastic-net estimates of the NBR to be a solution of \eqref{eq:en} is
\begin{enumerate}
\item
${\hat \beta _k}={\hat \beta _k} \ne 0$ if
$\frac{1}{n}\sum\limits_{i = 1}^n {{x_{ik}}} \frac{{\theta ({e^{\textit{\textbf{X}}_i^T\boldsymbol{\hat \beta} }} - {Y_i})}}{{\theta  + {e^{{\textit{\textbf{X}}_i^T}\boldsymbol{\hat \beta} }}}} =[{\rm{sign}}{\hat \beta _k}]({\lambda _1} + 2{\lambda _2} | {{{\hat \beta }_k}} |).$
\item
${\hat \beta _k}= 0$ if
$\left| \frac{1}{n}{\sum\limits_{i = 1}^n {{x_{ik}}} \frac{{\theta ({e^{{\textit{\textbf{X}}_i^T} \boldsymbol{\hat \beta} }} - {Y_i})}}{{\theta  + {e^{{\textit{\textbf{X}}_i^T} \boldsymbol{\hat \beta} }}}}} \right| \le {\lambda _1}.$
\end{enumerate}
\end{lemma}

\cite{Zhou2013} gave an elementary proof of KKT conditions for the elastic-net penalized optimization problem in a linear regression. Note that the KKT conditions are a standard result of sub-differentiation techniques. The prerequisite ${\lambda _2} > 0$ in Lemma~\ref{lem:iff} is indispensable. The reason is that we need ${\lambda _2} > 0$, such that $F(\hat \bfbeta  + \varepsilon {\textbf{e}_k};{\lambda _1},{\lambda _2}) - F(\hat \bfbeta ;{\lambda _1},{\lambda _2})>0$ where $\{{\textbf{e}_k}\}_{k=1}^p$ are unit coordinate vectors, see Appendix S2. Then $\hat \bfbeta$ is the unique local minimum. The KKT conditions are crucial for all sections below.

\subsection{$\ell_{q}$-estimation error using a compatibility factor}
\label{compatibility}

This section presents the sparse estimator for a high-dimensional NBR by using the fact that the elastic-net estimator is asymptotically close to the true parameter under some suitable regularity conditions.

For fixed designs $\{\textit{\textbf{X}}_i\}_{i=1}^n$, let ${\bfbeta ^*}$ be the vector of true coefficients, which satisfies
\begin{equation}\label{eq:oracle1}
{\rm{E}}{Y_i}={e^{\textit{\textbf{X}}_i^T\boldsymbol{ \beta}^* }}.
\end{equation}
In some sense, we can never really know the expectation of the negative log-likelihood, because ${\bfbeta ^*}$ is the unknown parameter in the functional estimating equation ${\textit{\textbf{X}}_i^T\boldsymbol{ \beta}^* }=\log ({\rm{E}}{Y_i})$.

In high-dimensions, we are interested in the sparse estimates defined in (\ref{eq:en}) by adding elastic-net penalty. For the true coefficient vector ${\bfbeta^{\rm{*}}} =({\beta_1^{\rm{*}}},\cdots, {\beta_p^{\rm{*}}})^T$, let $H = \{ j:{\beta_j^{\rm{*}}} \ne 0, j=1,\cdots,p\}$ and ${H^c} = \{ j:{\beta_j^{\rm{*}}}= 0, j=1,\cdots,p\} $ be the nonzero and zero components, respectively. Let $ {d_H^*} = \left| H \right|$ be the number of nonzero coefficients in ${\bfbeta ^*}$, i.e. the support of ${\bfbeta ^*}$. For any $\bm{b}\in {\mathbb{R}^p}$ and index set $H \in \{1,2,\cdots,p\}$, define the sub-vector indexed by $H$ as $\bm{b}_H = (\cdots ,{\tilde b_j}, \cdots )^T \in {\mathbb{R}^p}$, with ${\tilde b_j}=  b_j$ if $j \in H$, and $\tilde  b_j = 0$ if $j \notin H$. In the MLE theory, we know that the Kullback--Leibler (K--L) divergence measures how one probability distribution is different from another, based on a quasi-distance of two log-likelihoods. Similarly, in order to measure the derivative discrepancy between two penalized log-likelihood function w.r.t. the parameters, the \emph{symmetric Bregman (SB) divergence} between $\ell(\bfbeta_1 )$ and $\ell(\bfbeta_2 )$  is
\[{D_g^s}(\bfbeta_1 ,\bfbeta_2 ) = {(\bfbeta_1  - \bfbeta_2 )^T}[\dot \ell_n( \bfbeta_1 ) - \dot \ell_n(\bfbeta_2 )+{\lambda _2}(\dot g( \bfbeta_1) - \dot g(\bfbeta_2 ))],~\bfbeta_1 ,\bfbeta_2 \in \mathbb{R}^p.\]
If $g=0$, the symmetric Bregman divergence is ${D^s}(\hat \bfbeta ,\bfbeta ) = {(\hat \bfbeta  - \bfbeta )^T}[\dot \ell_n(\hat \bfbeta ) - \dot \ell_n(\bfbeta )]$. In this case, the symmetric Bregman divergence is a type of generalized quadratic distance (Mahalanobis distance), which can been viewed as a symmetric extension of the K--L divergence. See \cite{Nielsen09} and \cite{Huang13} for more details about SB divergence. Because $g(\bfbeta)$ is a nonnegative convex function, we have the inequality: ${D_g^s}(\bfbeta_1 ,\bfbeta_2 ) \ge {D^s}(\bfbeta_1 ,\bfbeta_2 ).$

The key to derive the oracle inequalities also depends on the behavior of the Hessian matrix of the NBR:
$\ddot \ell_n(\bfbeta )=\frac{1}{n} \sum\limits_{i = 1}^n {{\tilde { \textit{\textbf{X}}}_i}{\tilde{\textit{\textbf{X}}}_i^T}},$
where ${\tilde{\textit{\textbf{X}}}_i}:={\textit{\textbf{X}}_i}{(\frac{{\theta (\theta  + {Y_i}){e^{{\textit{\textbf{X}}_i^T}\boldsymbol{\beta} }}}}{{{{(\theta  + {e^{{\textit{\textbf{X}}_i^T}\boldsymbol{\beta} }})}^2}}})^{1/2}}$ is the \emph{curvature-scaled design}.

In the fixed design linear model ${\rm{E}}{\bm Y} = {\mathbf X}\bm\beta^*$ with ${\rm{Var}}{\bm Y}=\mathbf{I}_p \sigma^2$, it can be shown that, with probability greater than $1-{\delta _n}$,
\begin{align}\label{olsl2}
\|\hat{\bm\beta}_{LS}-{\bm\beta}^*\|_2 \le \sigma \sqrt {\frac{p}{n}}\cdot  [{\delta _n}{\lambda _{\min }}( {{\textstyle{1 \over n}}{{\bf{X}}^T}{\bf{X}}})]^{-1/2}.
\end{align}
for the ordinary least square (OLS) estimator $\hat {\boldsymbol{\beta}}_{LS}=(\mathbf{X}^T\mathbf{X})^{-1}\mathbf{X}^T\textit{\textbf{Y}}$, see Section 8.1 of \cite{Zhang20}. In an increasing dimension $p=p(n)$, it is well--known that the \emph{Gram matrix} is $\frac{1}{n} \sum\limits_{i = 1}^n {{\textit{\textbf{X}}_i}{\textit{\textbf{X}}_i^T}}$ (i.e., the correlation matrix between the covariates), which is singular when $p>n$. The positivity assumption of the ${\lambda _{\min }}( {{\textstyle{1 \over n}}{{\bf{X}}^T}{\bf{X}}})$ is crucial to obtain optimal convergence under $p<\infty$. In the sparse high-dimensional linear model via Lasso, to obtain the oracle inequality with the fast and optimal rate as discussed in \cite{Bickel09}, the following versions of the restricted minimal eigenvalue is usually needed under sparse cone set \eqref{eq:re}.

Let the \emph{sparse cone set} be
\begin{equation}\label{eq:re}
{\mathop{\rm S}\nolimits} (s ,H) := \{ \bm b \in {\mathbb{R}^p}:{|| {{{\bm b}_{{H^c}}}} ||_1} \le s {|| {{{\bm b}_H}} ||_1}\},~(s \in \mathbb{R}^+).
\end{equation}
The \emph{compatibility factor} (denoted by $C(s ,H,\mathbf\Sigma )$; see \cite{Geer2007}) of a $p \times p$ nonnegative-definite matrix $\mathbf \Sigma$ is defined by
\begin{equation}\label{eq:comp}
C^2(s ,H,\mathbf\Sigma ) := \mathop {\inf }\limits_{0 \ne {\bm b} \in {\rm{S}}(s ,H)} \frac{{{d_H^*}{{({{\bm b}^T}\mathbf\Sigma {\bm b})}}}}{{{{\|{{{\bm b}_H}} \|_1^2}}}} > 0,~~(s  \in \mathbb{R}).
\end{equation}

To derive the $\ell_{q}$-loss ($q>1$) oracle inequalities for the target coefficient vectors, we require the concept of \emph{weak cone invertibility factors} (weak CIF; see (53) of \cite{Ye10}),
\begin{equation}\label{eq:wk}
{C_q}(s ,H,\mathbf\Sigma ): = \mathop {\inf }\limits_{0 \ne {\bm{b}} \in {\rm{S}}(s ,H)} \frac{{{{{d_H^*}}^{1/q}}( {{\bm{b}}^T}\mathbf\Sigma {\bm{b})}}}{{{{||{{{\bm{b}}_H}} ||_1}}\cdot{{|| {\bm{b}} ||}_q}}} > 0,(s \in \mathbb{R}).
\end{equation}
This constant generalizes the compatibility factor, and is close to the restricted eigenvalue; see \cite{Bickel09}.

From the results in \cite{Ye10} and \cite{Huang13}, we know that the positivity assumptions of compatibility factor and the weak CIF can achieve sharper upper bounds for the oracle inequalities because both are bigger than the restricted eigenvalue:
\begin{equation*}
{\rm{Re}}(s ,H,\mathbf\Sigma ) := \mathop {\inf }\limits_{0 \ne {\bm b} \in {\rm{S}}(s ,H)} \frac{{{{{{\bm b}^T}\mathbf\Sigma {\bm b}}}}}{{{{\|{{{\bm b}}} \|_2^2}}}} \le \mathop {\inf }\limits_{0 \ne {\bm b} \in {\rm{S}}(s ,H)} \frac{{{d_H^*}{{({{\bm b}^T}\mathbf\Sigma {\bm b})}}}}{{{{\|{{{\bm b}_H}} \|_1^2}}}}=C^2(s ,H,\mathbf\Sigma ),~(s  \in \mathbb{R}),
\end{equation*}
due to $\|{{\bm{b}}_H}\|_{1} \leq {{d_H^*}}^{1 / 2}\|{{\bm b}}\|_{2}$.

Using the definitions of SB divergence with $\bfbeta_1=\hat \bfbeta  ,\bfbeta_2={\bfbeta ^*}$, let ${z^*} := {\| {\dot \ell_n ({\bfbeta ^*}) + {\lambda _2}\dot g({\bfbeta ^*})} \|_\infty }$ and $\Delta := \hat \bfbeta  - {\bfbeta ^*}$. We now provide the lower and upper bounds for the symmetric Bregman divergence.
\begin{lemma}[Theorem 1 in \cite{Yu2010}]\label{lem:ulb}
For the GLCP estimation, we have
\begin{align}\label{eq-msb>sb}
({\lambda _1} - {z^*})&{|| {{{\Delta }_{{H^c}}}} ||_1} \le {D_g^s}(\hat \bfbeta ,{\bfbeta ^*}) + ({\lambda _1} - {z^*}){|| {{{\Delta }_{{H^c}}}}||_1} \le ({\lambda _1} + {z^*}){|| {{{\Delta }_H}} ||_1}.
\end{align}
\end{lemma}
If $  {z^*} \le \frac{{\zeta  - 1}}{{\zeta  + 1}}{\lambda _1} $, for some $\zeta  > 1$, the inequality \eqref{eq-msb>sb} imply
\begin{equation} \label{eq:ulb2}
\frac{{2{\lambda _1}}}{{\zeta  + 1}}|| {{{\Delta }_{{H^c}}}} ||_1 \le D_{\rm{g}}^s(\hat \bfbeta ,{\bfbeta ^*}) + \frac{{2{\lambda _1}}}{{\zeta  + 1}}|| {{{\Delta }_{{H^c}}}} ||_1 \le \frac{{2\zeta {\lambda _1}}}{{\zeta  + 1}}|| {{{\Delta }_H}} ||_1,
\end{equation}
from ${\lambda _1} - {z^*} \ge \frac{{2{\lambda _1}}}{{\zeta  + 1}}$ and ${\lambda _1} + {z^*} \le \frac{{2\zeta {\lambda _1}}}{{\zeta  + 1}}$. By (\ref{eq:ulb2}), we have
\begin{equation} \label{eq-con}
{|| {{\Delta}_{{H^c}}} ||_1} \le \zeta{\| {{{\Delta }_H}} \|_1}.
\end{equation}
Hence we conclude that in the event
\begin{center}
${\cal K}_\lambda: = \left\{ {{z^*}= {\| {\dot \ell_n ({\bfbeta ^*}) + {\lambda _2}\dot g({\bfbeta ^*})} \|_\infty } \le \frac{{\zeta  - 1}}{{\zeta  + 1}}{\lambda _1}} \right\},$
\end{center}
the error of estimate $\Delta= \hat \bfbeta  - {\bfbeta ^*} \in {\mathop{\rm S}\nolimits} (\zeta ,H)$. Then assumptions $C^2(s ,H,\mathbf\Sigma )>0$ and ${C_q}(s ,H,\mathbf\Sigma )>0$ for the Hessian matrix $\mathbf\Sigma=\ddot \ell_n(\bfbeta^*)$ are indispensable assumptions for deriving the targeted oracle inequalities from the optimization \eqref{eq:en} and the expected version \eqref{eq:oracle1}. Some additional regularity conditions are required.
\begin{itemize}
\item [\textbullet] (C.1): Assume bounded covariates,
\begin{center}
$\max \{ \left| {x_{ij}} \right|; 1 \le i \le n,1 \le j \le p\}=L  < \infty $.
\end{center}

\item [\textbullet](C.2):  Based on the covairates $\{\textit{\textbf{X}}_i\}_{i=1}^n$, we assume identifiability condition that $\boldsymbol{\beta}\in \mathbb{R}^p$ satisfies
    \begin{center}
        ${{\textit{\textbf{X}}_i^T}(\boldsymbol{\beta} + \bfdelta)}={{\textit{\textbf{X}}_i^T} \boldsymbol{\beta}}$ implies ${{\textit{\textbf{X}}_i^T}\boldsymbol{\bfdelta}} = 0$ for $\bfdelta \in {\mathbb{R}^p}$.
    \end{center}

\item [\textbullet](C.3): Suppose that $||\bfbeta^* ||_1 \le B$.
\end{itemize}
The bounded covariates in C.1 are a common assumption in GLMs (see Example 5.40 of \cite{Vaart1998}); it may be achieved by performing a bounded and monotone transformation of the covariates in the real data. The identifiability condition C.2 and the compact parameter space C.3 are common assumptions for obtaining the consistency for a general M-estimation; see Section 5.5 and the remark after Theorem 5.9 in \cite{Vaart1998}. Recently, \cite{WeiBbach2019} showed the consistency of the NBR with fixed covariates, under the assumption that all possible parameters and the regressor are in some compact spaces.

First, we present the nonasymptotic upper bounds for the elastic-net regularized NBR in the following two theorems.
\begin{theorem}\label{theo-oi}
Let  $C(\zeta ,H,\ddot \ell_n({\bfbeta ^*}))$ and $C_q(\zeta ,H,\ddot \ell_n({\bfbeta ^*}))$ be the compatibility factor and the weak cone invertibility factor, respectively, defined above. Define $\tau  :=\frac{{L(\zeta + 1){d^*}{\lambda _1}}}{{2{{\left[ {C(\zeta ,H)} \right]}^2}}}  \leq \frac{1}{2}e^{-1}$. Assume that  (C.1), (C.2), and the event ${\cal K}_\lambda$ hold. Then, we have
\begin{equation}\label{eq:ci-nbd}
\begin{aligned}
 \| {\hat \bfbeta  - {\bfbeta ^{\rm{*}}}} \|_1 \le \frac{{{e^{{2a_\tau }}}(\zeta  + 1){d_H^*}{\lambda _1}}}{{2{ {C^2(\zeta ,H,\ddot \ell_n({\bfbeta ^*}))} }}} ~~\text{and}~~\|\hat \bfbeta  - {\bfbeta ^*}\|_{q} \le \frac{{2{e^{2{a_\tau }}}\zeta {{d_H^*}^{1/q}}{\lambda _1}}}{{\left( {\zeta  + 1} \right){C}_q(\zeta ,H,\ddot \ell_n({\bfbeta ^*}))}},
\end{aligned}
\end{equation}
where ${a_\tau } \le \frac{1}{2}$ is the smaller solution of the equation $a {e^{ - 2a }} = \tau$.
\end{theorem}

On the one hand, the Theorem~\ref{theo-oi} contains basic oracle inequalities conditioning on the random event, which needs further refinements. What remains to be done is to focus the probability upper bound of event ${\cal K}_\lambda$. With assumption (C.3), we have ${z^*} \le {\| {\dot \ell_n ({\bfbeta ^*})} \|_\infty } + 2{\lambda _2}B$. Our aim of proof is to have
\begin{equation} \label{eq:rela}
P({\cal K}_\lambda^c) \le P({\| {\dot \ell_n ({\bfbeta ^*})} \|_\infty } \ge \frac{{\zeta  - 1}}{{\zeta  + 1}}{\lambda _1}- 2{\lambda _2}B ) \to 0~~\text{as}~ n,p \to \infty,
\end{equation}
provided that ${\lambda _2}$ is sufficient small.

To bound ${\| {\dot \ell_n ({\bfbeta ^*})} \|_\infty }$, all we need is to apply some concentration inequalities in terms of the NB empirical processes \eqref{eq:NBEP}, that is the sum of independent weighted centralized NB random variables. Because the dispersion parameter $\theta$ is known, the NB random variables $\{ {Y_i}\} _{i = 1}^n$ belong to the exponential family
\begin{equation} \label{eq:E-F}
f({y_i};{\eta _i}) \propto  \exp \{ {y_i}{\eta _i} - \psi ({\eta _i})\}~\text{with}~{\eta _i}: = \bm X_i^T{\bfbeta ^*} + \log (\theta  + {e^{{\bm X_i^T}{\boldsymbol{\beta} ^*}}}) \in \Theta ,
\end{equation}
where $\Theta$ is the compact parameter space. Thus, under fixed design, the sub-Gaussian concentration inequalities for the non-random weighted sum of exponential family random variables with compact parameter space is applicable; see Lemma 6.1 in \cite{Rigollet12} or Proposition 3.2 in \cite{Zhang20} with more discussion.

On the other hand, the Compatibility Factor and weak CIR we employ in this section are random constants. They contains the Hessian matrix of the true coefficient vector, and thus encapsulate the random quantities $\{ {Y_i}\} _{i = 1}^n$. Note that deriving the lower bound for these random quantities decreaseS the probability that oracle inequalities are true, but the loss is negligible in the next theorem. Next, we successfully show using the NB concentration inequality that a reasonable non-random lower bounds of the compatibility factor (or the weak CIR) makes sure that the upper bounds are constants with high probability. Thus the rigorous convergence rate of $\hat \bfbeta$ is well established. Note that \cite{Yu2019} directly assume that the inverse of compatibility factor of $\ddot{\ell}(\bfbeta^{*})$ for the Cox model is $O_p(1)$, which they call it ``a high-level condition''. The Hessian matrix of the Cox model is also a random element.

Two events for truncating the compatibility factor and the weak CIR, is defined by
$${\cal E}_c:=\{{C^2(\zeta ,H,\ddot \ell_n({\bfbeta ^*}))}>{C_t^2(\zeta ,H)}\}~\text{and}~{\cal E}_w:=\{{C_{q}}(\zeta ,H,\ddot \ell_n({\bfbeta ^*}))>{C_{qu}}(\zeta ,H)\},$$
where ${C_t^2(\zeta ,H)}$ and ${C_{qu}}(\zeta ,H)$ are nonrandom constants defined in the proof for certain constants $t, u>0$.
\begin{theorem}\label{thm:event}
Under the assumptions of Theorem~\ref{theo-oi}, we further assume (C.2). Let $B_1$ be the constant satisfying $C_{\xi ,B_1} := \frac{{\zeta  - 1}}{{\zeta  + 1}} - 2{B_1} > 0$. Let ${\lambda _1} = \frac{{C_{LB}^{}L}}{C_{\xi ,{B_1}}}\sqrt {\frac{{2r\log p}}{n}} $, where ${C_{LB}^2}:={e^{LB}} + \frac{{{e^{2LB}}}}{\theta }$ is a variance-depending constant and $r>1$ is a constant. Put ${\lambda _2}= B_1{\lambda _1}/B$. Under the event ${\cal K} \cap {\cal E}_c$ (or ${\cal K} \cap {\cal E}_w$), we have:
\begin{align}\label{eq:ci-nbd1}
 &P\left( \| {\hat \bfbeta  - {\bfbeta ^{\rm{*}}}} \|_1 \le \frac{{{e^{{2a_\tau }}}(\zeta  + 1){d_H^*}{\lambda _1}}}{{2{ {C_t^2(\zeta ,H)} }}}\right)\ge 1 - \frac{2}{{{p^{r - 1}}}} - 2{p^2}e^{  - \frac{{n{t^2}}}{{2{{[d_H^*C_{LB}(1 + \varsigma )L^2]}^2}}}}\\
 &\text{or}~P\left( \|\hat \bfbeta  - {\bfbeta ^*}\|_{q} \le \frac{{2{e^{2{a_\tau }}}\zeta {d_H^*}^{1/q}{\lambda _1}}}{{\left( {\zeta  + 1} \right){C}_{qu}(\zeta ,H)}}\right)\ge 1 - \frac{2}{{{p^{r - 1}}}} - 2{p^2}e^{  - \frac{{n{u^2}}}{{2{{[d_H^*C_{LB}(1 + \varsigma )L^2]}^2}}}}.
\end{align}
\end{theorem}

If we presume the condition  ${d_H^*} = O(1)$ in Theorem \ref{thm:event}, which implies that the error bound is of  order $\sqrt {\frac{{\log p}}{n}} $, the elastic-net estimates have $\ell_{1}$-consistency property when the dimension of covariates increases with order $e^{o(n)}$. The MLE has the convergence rate $\frac{1}{{\sqrt n }}$. Nevertheless, in high-dimensional condition, we have to magnify $\sqrt {\log p} $ to the convergence rate of MLE. If we assume ${d_H^*} = o(\sqrt {\frac{n}{{\log p}}} )$, that is $p = e^{ o(n/{d_H^*})} $, then ${d_H^*}\lambda  = o(1)$ which implies the consistency property. If we consider random designs, the story is different. Our purpose in next section is to present an approach that avoids the random upper bound for the $\ell_{1}$ or $\ell_{2}$ estimation error, and provides the oracle inequality for the squared prediction error.

\subsection{The prediction error under a random design}
\label{Stabil}
In this section, we focus on the prediction error. We assume that the $n \times p$ design matrix $\textbf{X} = {({\emph{\textbf{X}}_1}, \cdots ,{\emph{\textbf{X}}_n})^T}$ is random. In our applications, the test data set is a new design $\textbf{X}^*$, which is an independent copy of $\textbf{X}$. Thus it requires the randomness assumption of the design matrix. We aim to predict the response $Y_{n+1}$ using the new random covariates $\textit{\textbf{X}}_{n+1}$ by resorting to elastic-net estimator $\hat \bfbeta $ to estimate the unknown $Y_{n+1}$.

Here $\emph{\textbf{Y}}\in \mathbb{R}^{n}$ contains $n$ independent ({\rm{ind.}}) responses $\{ {Y_i}\} _{i = 1}^n$. Thus the covariates and responses are considered pairs of random vectors $(\textbf{X} ,\emph{\textbf{Y}})$. When $\{\textit{\textbf{X}}_i\}_{i=1}^n$ is degenerately distributed, it reduces to a fixed design, and hence the result here also holds for a fixed design. Through this paper, we denote the element in the design matrix $\{x_{ij}\}$ as fixed design, and $\{X_{ij}\}$ as random design. The conditional distribution of a single observation $Y_i\vert \textit{\textbf{X}}_i = \textit{\textbf{x}}_i$ is assumed to be conditional NB distributed with ${\rm{E}}(Y_i|\textit{\textbf{X}}_i = \textit{\textbf{x}}_i) =e^{\textit{\textbf{x}}_i^T\boldsymbol{\beta}}$.

Let ${\bfbeta ^*}$ be the vector of true coefficients, which is defined by the minimizer
\begin{equation}\label{eq:oracle2}
{\bfbeta^{\rm{*}}}{=}\mathop {\rm{argmin} }\limits_{\boldsymbol{\beta}\in {{\mathbb{R}}^p}} {\rm{E}}l(Y,{\bm{X}, {\bfbeta }}),
\end{equation}
where $l(Y,\bm{X}, \bfbeta) = Y\bm{X}^T\bfbeta  - (\theta  + Y)\log (\theta  + {e^{\bm{X}^T\boldsymbol{\beta} }})$ is the NB loss.

 To derive nonasymptotical bounds for the $\ell_{1}$ estimation and square prediction error, we focus on the empirical process for any possible $\bfbeta$ [on the NB loss function in \eqref{eq:oracle2} with random $\textbf{X}$],
$$\mathbb{P}_n l(\textit{\textbf{X}} ,Y, {{\bfbeta }}): =  - \frac{1}{n}\sum\limits_{i = 1}^n {[{Y_i}{{{\textit{\textbf{X}}_i^T}\boldsymbol{\beta} }}  - (\theta  + {Y_i})\log (\theta  + {e^{{\textit{\textbf{X}}_i^T}\boldsymbol{\beta} }})]},$$
where $\mathbb{P}_n$ is the empirical measure of the samples $\{(\textit{\textbf{X}}_i, Y_i)\}_{i=1}^{n}  \stackrel{{\rm{ind.}}}{\sim} (\textit{\textbf{X}} ,Y)$.

The concentration and fluctuation of the empirical process are crucial to evaluating the consistent properties of the estimates. The proof oracle inequalities in this section consists 3 steps, including: 1. Checking $\bm{\hat \beta}  - {\bm\beta ^*}$ be in cone set by using definition of penalized estimation and KKT-like conditions; 2. Verifying the high probability of KKT-like conditions; 3. Deriving the oracle inequalities from restricted eigenvalue condition with some elementary inequalities. For simplicity, we use symbols for the empirical process in this section. We need some assumptions, such that $\hat \bfbeta $ is consistent.
 \begin{itemize}
\item [\textbullet] (H.1): All variables $\textit{\textbf{X}}_i$ are bounded: there exists a constant $L>0$, such that
    $|||{\bf X}|||_\infty:=\mathop {\sup }\limits_{1 \le i \le \infty} {\| \textit{\textbf{X}}_i \|_\infty } \le L$ a.s.

\item [\textbullet](H.2): Assume that $||\bfbeta^* ||_1 \le B$.

\item [\textbullet](H.3): There exists a large constant ${M_0}$, such that $\hat\bfbeta$ is in the $\ell_1$ ball:
\begin{center}
$\hat\bfbeta \in {{\cal S}_{M_0}}(\bfbeta ^*):= \left\{ {\bfbeta  \in {\mathbb{R}^p}:{\|{\bfbeta } - \bfbeta ^*\|_1}  \le {M_0}} \right\}$.
\end{center}
\item [\textbullet](H.4):  Let ${\theta }>1$. The negative log-density of $n$ independent NB responses $\psi(\bm y) :=-\log p_{\bm Y}(\bm y)$, for $\bm Y = {({Y_1}, \cdots ,{Y_n})^T}$, satisfies the \emph{strongly midpoint log-convex} properties for some $\gamma>0$,
\begin{equation}\label{eq:SDCM}
 \psi(\bm x)+\psi(\bm y)-\psi(\lceil \frac{1}{2} \bm x+\frac{1}{2} \bm y\rceil)-\psi(\lfloor\frac{1}{2} \bm x+\frac{1}{2} \bm y\rfloor) \ge \frac{\gamma}{4}\|\bm x-\bm y\|_{2}^{2}\quad \forall \bm x, \bm y \in \mathbb{Z}^{n}.
\end{equation}
\end{itemize}
\begin{remark}
(H.1) and (H.2) are mentioned in \cite{Blazere14}, and (H.3) is a high technique condition owing to the noncanonical link GLMs. The constraint in the optimization \label{eq:en} is equivalent to ${\alpha }\left\| \boldsymbol{\beta}   \right\|_1 + (1-\alpha){\left\| \boldsymbol{\beta}  \right\|_2^2}\le t$, with unknown $\alpha  \in [0,1]$ and $t \in \mathbb{R}$ leading to $\| \hat\bfbeta\|_1 \le {M_0}$ if we suppose that $t/\alpha  \le {M_0}$. There is a constant $K>0$, such that $\mathop {\max }\limits_{1 \le i \le n} \left| {\textit{\textbf{X}}_i^T\bfbeta^*} \right| \le K$ a.s., for all $n$. A convex function $F$ is called strongly convex if the Hessian matrix of $F$ has a (uniformly)
lower bounded eigenvalue. While examining exponential families in high dimensions, \cite{Kakade10} assumed that continuous exponential families \eqref{eq:E-F} have strongly convex log-likelihood function with ${\eta _i}$ in a sufficiently small neighborhood. For a fixed dimensional MLE, \cite{Balabdaoui13} show that the discrete log-concave maximum likelihood estimator is strongly consistent under some settings. Our assumption (H.4) is a condition that ensures that the suprema of the multiplier empirical processes of $n$ independent responses have sub-Gaussian concentration phenomena in (S1.19), which can be alternatively be checked by the tail inequality for suprema of empirical processes corresponding to classes of unbounded functions \citep{Adamczak2008}.  For the case of a fixed design in Section \ref{compatibility}, we do not require (H.4) in order to derive the oracle inequalities.
\end{remark}

In this section, we give sharp bounds for $\ell_{1}$ estimation and squared prediction errors for NBR models by looking for a weaker condition that is analogous to the restricted eigenvalue (RE) condition proposed by \cite{Bickel09}, and the weak CIF and compatibility factor conditions presented in Section 3.2. Here, we borrow a condition from the Stabil Condition introduced by \cite{Bunea08} for $\ell_{1}$ and $\ell_{1}+\ell_{2}$ penalized logistic regressions.

For $c$, $\varepsilon>0$, we define the \emph{fluctuated cone set} for some bias vector $ \bm{b}$ as
\begin{equation}\label{eq:re1}
{\mathop{\rm V}\nolimits} (c ,\varepsilon, H) := \{ \bm{b} \in {\mathbb{R}^p}:{||{{\bm{b}_{{H^c}}}} ||_1} \le c{|| {{\bm{b}_H}}||_1  } + \varepsilon \},
\end{equation}
which is a fluctuated (or measurement error) version of the cone set ${\mathop{\rm S}\nolimits} (s ,H) := \{ \bm b \in {\mathbb{R}^p}:{|| {{{\bm b}_{{H^c}}}} ||_1} \le s {|| {{{\bm b}_H}} ||_1}\}$ mentioned in \eqref{eq:re}.

We substitute $\bm{b}=\hat\bfbeta- {\bfbeta ^*}$ into the proof. For real data, let $\hat\bfbeta$ be the estimator based on the true covariates, and let $\hat\bfbeta_{me}$ be the estimator from covariates with a measurement error. Note that under the cone condition ${||{{\bm{b}_{{H^c}}}} ||_1} \le c{|| {{\bm{b}_H}}||_1  }$, for $\bm{b}=\hat\bfbeta- {\bfbeta ^*}$, we get
\begin{align*}
|| {{(\hat\bfbeta_{me}- {\bfbeta ^*})_{{H^c}}}}||_1 -||(\hat\bfbeta  -\hat\bfbeta_{me})_{{H^c}}||_1&\le {||{{(\hat \bfbeta  - {\bfbeta ^*})_{{H^c}}}} ||_1}\le c{||{{(\hat \bfbeta  - {\bfbeta ^*})_{{H}}}} ||_1}\\
 &\le c|| {{(\hat\bfbeta_{me}- {\bfbeta ^*})_H}}||_1+c||(\hat\bfbeta  -\hat\bfbeta_{me})_H||_1.
\end{align*}
Then,
\begin{align*}
{||{{\bm{b}^{me}_{{H^c}}}} ||_1} \le c{|| {{\bm{b}^{me}_H}}||_1  } + \varepsilon~~\text{for}~\bm{b}^{me}:=\hat\bfbeta_{me}- {\bfbeta ^*},
\end{align*}
where $\varepsilon=c|| {{(\hat\bfbeta_{me}- {\bfbeta ^*})_H}}||_1+||(\hat\bfbeta  -\hat\bfbeta_{me})_{{H^c}}||_1$. This argument indicates that the fluctuated cone set quantifies the level of the measurement error if $\hat\bfbeta_{me}$ is misspecified as $\hat\bfbeta$.

On the fluctuated cone set, we assume that the $p \times p$ the expected empirical covariance matrix $\boldsymbol{\Sigma}={\rm{E}}\bm{X}\bm{X}^T$ fulfills the Stabil Condition as below. The Stabil Condition for matrix $\boldsymbol{\Sigma}$ avoids the random Hessian matrix in the Compatibility Factor Condition and the weak CIF Condition. However, there is no free lunch. The proposed oracle inequalities in this section require (H.4), which serves for the tail inequality for the suprema of NB empirical processes.

\begin{definition}\label{eq:def1}
(Stabil Condition with measurement error) For given $c,\varepsilon >0$, the matrix $\boldsymbol{\Sigma}$ satisfies the Stabil Condition $S(c,\varepsilon,k)$ if there exists $0<k<1$, such that
\begin{center}
${\textbf{b}^T}\boldsymbol{\Sigma}\textbf{b} \ge k||{\textbf{b}_H}||_2^2 - \varepsilon$
\end{center}
for any $\textbf{b} \in V(c,\varepsilon,H)$. Here, the restriction  $0<k<1$ can be attained by scaling $\textit{\textbf{X}}$.
\end{definition}

Let ${l_1}(\bfbeta ): = {l_1}(\bfbeta ,\bm{X},Y) := - Y[{\bm{X}^T}\bfbeta  - {\rm{log}}(\theta + {\rm{ exp}}\{  {\rm{ }}{\textit{\textbf{X}}^T}\bfbeta \} )]$, which is a linear function of the response, and let ${l_2}(\bfbeta): = {l_2}(\bfbeta,\textit{\textbf{X}}) := \theta {\rm{log}}(\theta {\rm{ + exp}}\{{\textit{\textbf{X}}^T}\bfbeta \} )$, which is free of the response. The NB loss function $l(\bfbeta ,\bm{X},Y) = {l_1}(\bfbeta ,\bm{X},Y) + {l_2}(\bfbeta ,\bm{X})$ is thus decomposed into two parts. Let $\mathbb{P}l({{\bfbeta }}) := {\rm E}l(\bfbeta,\bm{X},Y)$ be the expected risk function, where the expectation is under the randomness of $(\textit{\textbf{X}} ,Y)$. We prefer the centralized empirical loss $\left( \mathbb{P}_{n}-\mathbb{P}\right)  l(\bfbeta)$, which represents the fluctuation between the expected and the sample loss, rather than the loss itself. We break down the empirical process into two parts:
\begin{equation}\label{eq:EP}
\left( \mathbb{P}_{n}-\mathbb{P}\right)  l(\bfbeta)=\left( \mathbb{P}_{n}-\mathbb{P}\right) l_{1}(\bfbeta)+\left( \mathbb{P}_{n}-\mathbb{P}\right)  l_{2}(\bfbeta).
\end{equation}

In the following, we give upper bounds for the first and second parts of the empirical process: $(\mathbb{P}_{n}-\mathbb{P})( l_{m}(\bfbeta^{*})-l_{m}(\hat{\bfbeta}))$, for $m=1,2$. We show that  $(\mathbb{P}_{n}-\mathbb{P})( l_{m}(\bfbeta^{*})-l_{m}(\hat{\bfbeta}))$ has \emph{stochastic Lipschitz properties} (see \cite{Chi10}) with respect to ${\|{{\hat \bfbeta }} - \bfbeta ^*\|_1}$. Let the $\ell_1$ ball be ${{\cal S}_{M_0}}(\bfbeta ^*):= \left\{ {\bfbeta  \in {\mathbb{R}^p}:{\|{\bfbeta } - \bfbeta ^*\|_1}  \le {M_0}} \right\}$, which is referred as the \textit{local set}. Then,

\begin{proposition}\label{prop:upbound1}
Let the centered responses be $\{Y_i^c:={Y_i} - {\rm{E}}{Y_i}\}_{i=1}^n$ and, (H.1)--(H.4) are satisfied. If
$\lambda _{1}\ge 4L(2{\tilde C_{LB}} + A\sqrt {2\gamma } )\sqrt {\frac{{2\log 2p}}{n}}$,($A\ge 1, {\tilde C_{LB}^2}:={e^{LB}} + \frac{{(1+\theta){e^{2LB}}}}{\theta }$),  define the event $\mathcal{A}$ for the suprema of the multiplier empirical processes as
$$\mathcal{A}:=\left\{ \mathop {\sup }\limits_{\boldsymbol{\beta}_1,\boldsymbol{\beta}_2 \in {{\cal S}_{M_0}}(\boldsymbol{\beta}^*)} \left| \frac{{ 1}}{n}\sum\limits_{i = 1}^n \frac{Y_i^c{\theta \textit{\textbf{X}}_i^T(\boldsymbol{\beta}_1-{\bfbeta ^*} )}}{{(\theta {\rm{ + exp}}\{ \textit{\textbf{X}}_i^T\boldsymbol{\beta}_2 \})\|\bfbeta_1  - {\bfbeta ^*}\|_1 }} \right| \le \frac{\lambda _{1}}{4} \right\}.$$
Then, we have $P(\mathcal{A}) \ge 1- (2p)^{-A^2}$. Moreover,
$$P\left\{(\mathbb{P}_{n}-\mathbb{P})( l_{1}(\bfbeta^{*})-l_{1}(\hat{\bfbeta}))\le \frac{{{\lambda _{1}}}}{4} {\|{{\hat \bfbeta }} - \bfbeta ^*\|_1}\right\}\ge 1- (2p)^{-A^2}.$$
\end{proposition}

This proposition indicates that the discrepancy between the first part of the empirical process and its expectation is bounded from above by the tuning parameter multiplied by the $\ell_{2}$ norm of the difference between the estimated vector and the target vector. The $\frac{{{\lambda _{1}}}}{4}$ can be seen as a Lipschitz constant of the first part of the centralized empirical process.

Similarly to $\mathcal{A}$ as a KKT-like condition, we provide a crucial lemma to bound the second part of the empirical process with responses. Let ${\nu _n}(\boldsymbol{\beta} ,{\boldsymbol{\beta} ^*}): = \frac{{({\mathbb{P}_n} - \mathbb{P})\left( {{l_2}({\boldsymbol{\beta} ^*}) - {l_2}(\boldsymbol{\beta} )} \right)}}{{ {\|{\boldsymbol{\beta}} - \boldsymbol{\beta} ^*\|_1}  + {\varepsilon _n}}}$ be the normalized second part of the empirical process, which is a random variable indexed by $\boldsymbol{\beta}$. Then we define the \emph{local stochastic Lipschitz constant} for a certain $M>0$,
\begin{center}
${Z_M}(\bfbeta^*):=\mathop {{\rm{sup}}}\limits_{\boldsymbol{\beta}\in {{\cal S}_M}(\boldsymbol{\beta}^*)} |{\nu _n}(\bfbeta ,{\bfbeta ^*})|,~\text{and a random event}~{\cal B} := \{ {Z_M}(\bfbeta^*)\le \frac{{{{\lambda _1}}}}{4}\},$
\end{center}
 where we bound the local stochastic Lipschitz constant using the rescaled tuning parameter $\frac{{{{\lambda _1}}}}{4}$. Moreover, by definition, we have
\begin{center}
$| {{\nu _n}( \hat  \bfbeta ,{\bfbeta ^*})} | \le \mathop {\sup }_{{{\cal S}_M}(\boldsymbol{\beta}^*)} | {{\nu _n}( \hat \bfbeta ,{\bfbeta ^*})} | \le \frac{{{\lambda_1}}}{4}$,
\end{center}
which gives following bound,
\begin{equation}\label{prop:upbound2}
| {({\mathbb{P}_n} -\mathbb{P})({l_2}({\hat \bfbeta }) - {l_2}(\bfbeta^* ))} | \le \frac{{{\lambda_1}}}{4}{\rm{(}} {\|{ \hat \bfbeta } -\bfbeta ^*\|_1}  + {\varepsilon _n}{\rm{)}}~~\text{on}~\mathcal{B},
\end{equation}
provided that $\hat\bfbeta \in {{\cal S}_{M}}(\bfbeta ^*)$.

According to the following lemma, in the event $\mathcal{A}\bigcap\mathcal{B}$, the estimator $\hat{\bfbeta}$ lies in a known neighborhood of the true coefficient vector $\bfbeta^{*}$.

\begin{lemma}\label{local}
Under (H.2), let $8B{\lambda _2}+4M=\lambda _1$,  we have
$${\|{\hat \bfbeta } - \bfbeta ^*\|_1}  \le 16||\bfbeta^* || + 2{\varepsilon _n}~~\text{on}~\mathcal{A}\bigcap\mathcal{B}.$$
\end{lemma}

The proof of Lemma~\ref{local} relies on the optimization \eqref{eq:en} and  the definition of the minimizer $\bfbeta ^*$ from the expected loss \eqref{eq:oracle2}. By Lemma~\ref{local}, on the event $\mathcal{A}\bigcap\mathcal{B}$, we immediately get $\hat \bfbeta  \in {{\cal S}_{ 16B + 2{\varepsilon _n}}}(\bfbeta ^*)$. Note that we assume that $\hat\bfbeta \in {{\cal S}_{M_0}}(\bfbeta ^*)$, for some finite $M_0>M={ 16B + 2{\varepsilon _n}}$ in (H.3). That is Lemma~\ref{local} sharpens $\hat\bfbeta$ in the $\ell_1$-ball ${{\cal S}_{M}}(\bfbeta ^*)$, whereas $\hat\bfbeta$ is originally assumed in the $\ell_1$ ball ${{\cal S}_{M_0}}(\bfbeta ^*)$. Therefore, the following probability analysis of the event $\mathcal{A}\bigcap\mathcal{B}$ is indispensable. The event $\mathcal{A}\cap \mathcal{B}$ associated with the empirical loss functions plays an important role in deriving the oracle inequalities for general loss functions, because we could bound the $\ell_1$ estimation error conditioning on event  $\mathcal{A}\bigcap\mathcal{B}$.  We now give the result that the event $\mathcal{A}\cap \mathcal{B}$ occurs with a high probability.

\begin{proposition}\label{prop-highprobability}
Let $M = 16B + 2{\varepsilon _n}$. Suppose $\hat\bfbeta \in {{\cal S}_{M_0}}(\bfbeta ^*)$, for $\infty>M_0>M$, and that (H.1)-(H.4) hold. If
\begin{equation}\label{eq:lambda}
{\lambda _1}\ge \mathrm{max}\left(\frac{{20\theta AML}}{{M + {\varepsilon _n}}}\sqrt {\frac{{2\log 2p}}{n}},4L(2{\tilde C_{LB}} + A\sqrt {2\gamma } )\sqrt {\frac{{2\log 2p}}{n}}\right),~ A\ge1,
\end{equation}
then $P({\cal A} \cap {\cal B}) \ge 1 - 2{(2p)^{ - {A^2}}}.$
\end{proposition}

The proof of Theorem~\ref{theo-nbl} is based on some lemmas in Appendix S1, which show that the event $\mathcal{A}\cap \mathcal{B}$ holds with a high probability. Judging from the above probability analysis, we can formulate the main result of this section that gives bounds for the estimation and prediction error because the target model is sparse, and $\log p$ is tiny compared to $n$. In particular, the oracle inequality of the estimation error is useful in the following sections.

\begin{theorem}\label{theo-nbl}
 Assume condition $S(3.5,{\varepsilon _n},k)$ and (H1)--(H4) hold. Let ${\lambda _1}$ be chosen by \eqref{eq:lambda} and $\lambda_2 \le \frac{{{\lambda _1}}}{{8B}}$. Then, under the event $\mathcal{A}\cap \mathcal{B}$, we have
$P(\hat{\bfbeta}-\bfbeta^{*} \in {\rm V}(3.5,\frac{\varepsilon_{n}}{2},H))\ge 1 - 2{(2p)^{ - {A^2}}}$ and
\begin{equation}\label{eq:l1}
P\left\{\|\hat \bfbeta  - {\bfbeta ^*}\|_1 \le \frac{{{{2.25}^2}{\lambda _1}{d_H^*}}}{{{a}k + {\rm{2}}{\lambda _2}}} + (1 + \frac{a}{{{\lambda _1}}}){\varepsilon _n}\right\}\ge 1 - 2{(2p)^{ - {A^2}}}.
\end{equation}
Moreover, let the test data $(\bm X^*, Y^*)$ be an independent copy of the training
data $(\bm X, Y)$, and denote ${\rm{E^*}}(\cdot):={\rm{E}}(\cdot|\bm X^*)$. Conditioning on the event $\mathcal{A}\cap \mathcal{B}$, the squared prediction error is
\begin{equation}\label{eq:l2}
{\rm{E^*}}{[{\bm X^{*T}}( \hat \bfbeta- {\bfbeta ^*} )]^2} \le \frac{{17.71875{d_H^*}\lambda _1^2}}{{a(ak + 2{\lambda _2})}} + (\frac{{4{\lambda _1}}}{a} + 3.5){\varepsilon _n},
\end{equation}
where $a: = \mathop {\min }\limits_{\{ |x| \le LM+K,|y| \le K\} } \{ \frac{1}{2}\frac{{\theta {{\rm{e}}^x}({{\rm{e}}^y} + \theta )}}{{{{[\theta  + {{\rm{e}}^x}]}^2}}}\}>0 $.
\end{theorem}

Comparing with the upper bounds under the Compatibility Factor Condition in Section~\ref{compatibility}, in much the same fashion, we observe that when $d^{*}=O(1)$, the number of covariates increases by as much as $o(\exp(n))$. Then, the bound on the estimation error is $ o\left(1\right) $, and the elastic-net estimator ensures the consistent property. Theorem~\ref{theo-nbl} is also an improvement over Lemma~\ref{local} from a big neighborhood of ${\bfbeta ^*}$ to the desired small neighborhood of ${\bfbeta ^*}$.
\begin{remark}
Discussion of the measurement error ${\varepsilon _n}$ when $d_H^{*} <\infty$:
 \begin{itemize}
\item [\textbullet] 1. If ${\varepsilon _n} = o(\sqrt {\frac{{\log p}}{n}} )$, then
$
\|\hat \bfbeta  - {\bfbeta ^*}\|_1 \le  O(\sqrt {\frac{{\log p}}{n}} )
,~~{\rm{E^*}}{[{\bm X^{*T}}( \hat \bfbeta- {\bfbeta ^*} )]^2} \le  O( {\frac{{\log p}}{n}} )$;
\item [\textbullet] 2. If ${\varepsilon _n} = O(\sqrt {\frac{{\log p}}{n}} )$, then $\|\hat \bfbeta  - {\bfbeta ^*}\|_1 \le  O(1)$, but
$
{\rm{E^*}}{[{\bm X^{*T}}( \hat \bfbeta- {\bfbeta ^*} )]^2} \le  O(\sqrt {\frac{{\log p}}{n}} );$
\end{itemize}
 More typical examples for ${\varepsilon _n}$ are $\frac{1}{n}$ or even zero. Under the restricted condition $\hat{\bfbeta}-\bfbeta^{*} \in V(3.5,\dfrac{\varepsilon_{n}}{2},H)$, Case 2 tells us that if the order of fluctuations ${\varepsilon _n}$ is sightly lower than the order of the tuning parameter, elastic-net with $\lambda_2 \le \frac{{{\lambda _1}}}{{8B}}$ guarantees that the squared prediction error is asymptotically zero, with a lower rate $O(\sqrt {\frac{{\log p}}{n}} )$.
\end{remark}
%This coincides the prediction consistency obtained by \cite{Chatterjee2013}, he shows that under almost no assumptions the mean squared prediction error is $O(\sqrt {\frac{{\log p}}{n}} )$ under linear models.

\section{Applications of the oracles results}\label{sec-AP}
We now examine the non-asymptotic and asymptotic results. In this section, the applications are derived from oracle inequalities about the $\ell_1$ estimation error, and we assume that the design matrix is fixed, for simplicity.

\subsection{Grouping effect from oracle inequality}
\label{sec-GE}

\cite{Zou2005} show that the elastic-net has a grouping effect that asserts that strongly correlated predictors tend to be in or out of the model together when the coefficients have the same sign. \cite{Zhou2013} proves that the grouping effect of the elastic-net estimates holds without the assumption of the sign.  \cite{Yu2010} derives the asymptotical result of the grouping effect for elastic-net estimates of the Cox models. Based on the oracle inequalities we put forward, we provide an asymptotical version of the grouping effect inequality as $p,n \to \infty $ for the fixed design case.
\begin{theorem}\label{tm:ge}
Under the assumption of Theorem~\ref{thm:event} with $d_H^{*} <\infty$, suppose that the covariates (nonrandom) are standardized as
\begin{equation}\label{eq-sd}
 \frac{1}{n}\sum\limits_{i = 1}^n {x_{ij}^2 = 1},~\frac{1}{n}\sum\limits_{i = 1}^n {{x_{ij}} = 0,}~~\text{for}~~j = 1,2, \cdots ,p.
\end{equation}
Denote ${\rho _{kl}} = \frac{1}{n}\sum\limits_{i = 1}^n {{x_{ik}}{x_{il}}} $ as the correlation coefficient.  For any constant $E_s>0$, with probability at least $1 - \frac{2}{{{p^{r - 1}}}} - 2{p^2}e^{  - \frac{{n{t^2}}}{{2{{[d_H^*C_{LB}(1 + \zeta )L^2]}^2}}}} - \frac{{\sigma _n^2}}{{nE_s^2}}$,\\

(i). ${| {{{\hat \beta }_k} - {{\hat \beta }_l}}|^2} \le (1 - {\rho _{kl}})[{Ke^{2LM}}O(1) + \frac{1}{{\lambda _2^2}}(E_s + \mu_s )]$;\\

(ii). If the asymptotic correlation between two random predictors is asymptotically up to one, that is ${\rho _{kl}} = 1 - o(\lambda _2^2)$, with $\lambda _2^2 = O(\frac{{\log p}}{n}) \to 0$, we have
$$| {{{\hat \beta }_k} - {{\hat \beta }_l}}| \le \sqrt {{o_p}(1)[\lambda _2^2{e^{2LM}}O(1) + (E + \mu )]}.$$
\end{theorem}

This grouping effect oracle inequality asserts that if ${\rho _{kl}}$ tends to one with a high probability, the elastic-net is able to select covariates $k,l \in \{1,2,\dots,p\}$ together. Combined with the Lasso sparse estimation, the $\ell_1+\ell_2$ penalty enables strongly correlated predictors to be in or out simultaneously. In addition to the sparse estimation, intuitively, highly related covariates should have similar regression coefficients, but the Lasso cannot select them simultaneously.

\subsection{Sign consistency}\label{sec-IC0}
Sign consistency indicates whether an estimate is good, relating tp the estimated sign of the coefficient. A few researchers have studied the sign consistency property of the elastic-net. One condition for sign consistency is the \emph{Irrepresentable Condition} (IC). \cite{Zhao06} explore the IC for the sign consistency of a linear regression under a Lasso penalty. Moreover, the model selection consistency of elastic-net is studied by \cite{Jia2010}, following \cite{Zhao06}. Along the same line, for the elastic-net penalized Cox model, \cite{Yu2010} investigates the selection consistency. Here, the basic idea is that the KKT condition is a necessary and sufficient condition for the global minimizer of the target function. We focus on the elastic-net penalized NBR model's selection consistency based on some reasonable assumptions in a similar fashion. It is interesting to see that under the bounded covariates assumption, we do not need the IC, which is assumed in \cite{Yu2010} and \cite{Lv2018}. We rely only on the assumptions in Theorem~\ref{thm:event}.

\textbf{Uniform Signal Strength Condition}.
$$
{\beta _*}: =\mathop {\min }\limits_{j \in H} |\beta _j^*| \ge \frac{{{e^{{2a_\tau }}}(\zeta  + 1){d_H^*}{\lambda _1}}}{{2{ {C^2(\zeta ,H)} }}},
$$
with ${\lambda _1} = O(\sqrt {\frac{\log p}{n}} ), B{\lambda _2} = B_1{\lambda _1}$.

Assume ${d_H^*} < \infty $. \cite{Zhangch14} points out that the selection consistency theory characteristically necessitates a \emph{uniform signal strength condition} (or \emph{beta-min condition}) that the smallest nonzero regression coefficients ${\beta _*}: = \min \{ \left| {{\beta _j}} \right|:j \in H\}$ should be greater in size
than a thresholded level $O(\sqrt {\frac{{\log p}}{n}} )$. When ${\beta _*}$ is less than the level, the presence of weak signals cannot be detected by statistical inferences procedures.

\begin{theorem}\label{tm:sign}
Suppose that the uniform signal strength condition and the assumptions of Theorem~\ref{thm:event} hold. Let ${\lambda _1} = O(\sqrt {\frac{{\log p}}{n}} ),{d_H^*} < \infty $.
Then, for $\sqrt {\frac{{\log p}}{n}}  = o(1)$ and a suitable tuning parameter $r$ in Theorem~\ref{thm:event}, we have the following sign consistency:
\begin{equation}\label{eq:sc}
\mathop {\lim }\limits_{n,p \to \infty } P({\mathop{\rm sign}} \hat \bfbeta  = {\mathop{\rm sign}} {\bfbeta ^*}) = 1.
\end{equation}
\end{theorem}

\subsection{Honest variable selection and detection of weak signals}
\label{signals}

As a particular case of the random design in Section~\ref{Stabil}, we focus on the fixed design in this section, where the $\{\textit{\textbf{X}}_i\}_{i=1}^n$ is deterministic.

Recall that $\hat H: = \{ j:{{\hat \beta }_j} \ne 0\}$; thus $\hat H$ is an estimator of the true variable set $ H: = \{ j:{{\beta }_j} \ne 0\}$ (or the set of positives). Let ${\delta _1},{\delta _2}$ be constants such that $P(\hat H\not  \subset H) \le {\delta _1},P(H\not  \subset \hat H) \le {\delta _2}$. Then we have  $P(H \ne \hat H) \le P(\hat H\not  \subset H) + P(H\not  \subset \hat H) \le {\delta _1} + {\delta _2}$. If we treat $H$ as the null hypothesis, $P(\hat H\not  \subset H)$ is often called the false positive rate in the language of ROC curves (or type I error in statistical hypothesis testing; the estimate is $\hat H$ but it makes the decision $\hat H \subset H^c  $); $P(H\not  \subset \hat H)$ is often called the false negative rate (or type II error). Thus, the probability of correct subset selection under some random events $W$ (the assumptions hold with probability $P(W)$) is
\begin{equation}\label{eq:HP}
P(H = \hat H) \ge P(W) - {\delta _1} - {\delta _2}.
\end{equation}
From the $\ell_1$ estimation error obtained in Theorem~\ref{theo-nbl}, we easily bound the false negative rate $P(H \not\subset \hat H)$ in Proposition~\ref{cl:ci}. However, the upper bound of the false
positive rate $P(\hat H\not  \subset H)$ cannot be obtained directly, additional assumptions on the covariates correlation are required.

\begin{proposition}\label{cl:ci}
Let $\delta  \in (0,1)$ be a fixed number, and let the assumption of Theorem~\ref{theo-nbl} be satisfied. The weakest and strongest signal meet the condition: ${{B_0}}: = \frac{{{{2.25}^2}{\lambda _1}{d_H^*}}}{{ak + {\rm{2}}{\lambda _2}}} + (1 + \frac{a}{{{\lambda _1}}}){\varepsilon _n} \le \mathop {\min }\limits_{j \in H} |\beta _j^*|  \le B$. If $p = \exp \{ \frac{1}{{{A^2} - 1}}\log \frac{{{2^{1 - {A^2}}}}}{\delta }\} $, with $A>1$, then
$$P(H \subset \hat H) \ge P(\|\hat \bfbeta  - {\bfbeta ^*}\|_1 \le {B_0}) \ge 1 - \delta/p .$$
\end{proposition}

Note that the lower bound we have derived may be too large in some settings. For example, this may occur if ${d_H^*}$ is as large as ${\lambda _1}{d_H^*} = O(1)$ and $\mathop {\min }\limits_{j \in H} |\beta _j^*| \ge \frac{{{{2.25}^2}O(1)}}{{ak + {\rm{2}}{\lambda _2}}} = :D$, where $D$ is also a moderately large constant compared with the strongest signal threshold $B$. Then, we can only detect a few parts of the overall signals. To deal with this problem, we use a new approach (inspired by Section 3.1.2 in \cite{Bunea08}) to find a constant-free weakest signal detection threshold that relies only on the tuning parameter ${\lambda _1}$. Under some mild conditions on the design matrix, we show that the lower bounds can be sharpen considerably.

First, we assume that the covariates are centered and standardized as in (\ref{eq-sd}). This crucial method of processing covariates is also employed when studying the grouping effect in Section~\ref{sec-GE}. Second, let ${\rho _{kl}} = \frac{1}{n}\sum\limits_{i = 1}^n {{X_{ik}}{X_{il}}}$, for $k,l \in \{ 1,2,\cdots,p\} $ be the correlation constants between covarates $k$ and $l$. For a constant $h \in (0,1)$, we have the following condition.

\textbf{Identifiable Condition}:
$\mathop {\max }\limits_{k,l \in H,k \ne l} |{\rho _{kl}}| \le \frac{h}{{\theta {d_H^*}}},~~\frac{\theta }{n}\sum\limits_{i = 1}^n {X_{ik}^2}  = 1.$

This assumption of a maximal correlation constant of two distinct covariates on the true set $H$ measures the dependence structure using a constant $h$ in the whole predictor. A lower $h$ indicates a higher degree of separation, which makes it easier to detect weak signals. \cite{Bunea08} explained the intuition as follows:`` If the signal is very weak and the true variables are
highly correlated with one another and with the rest, one cannot hope to recover the true model with high probability". Interestingly, the grouping effect states that the elastic-net is able to simultaneously estimate highly correlated true variables, and this grouping effect is valid without the premise that the signal is enough strong. If both signals are faint under the level of the detection bounds, then the elastic-net estimates are both zero, and the grouping effect is also true.

Additionally, we require two conditions because we have to build some connections between $P(H\not  \subset \hat H),P(\hat H\not  \subset H)$ and the $\ell_{1}$-estimation error in Theorem~\ref{theo-nbl}. Let $a_i$~($b_i$) be the intermediate point between $\textit{\textbf{X}}_i^T\hat \bfbeta$ and $\textit{\textbf{X}}_i^T\bfbeta ^*$, by the first-order Taylor expansion of the function $f(t) = \frac{{{e^t}}}{{\theta  + {e^t}}}~(g(t) = \frac{1}{{\theta  + {e^t}}})$, and ${L_1},{L_2} \in [1,\infty )$. By (H.1)--(H.3), it leads to for all $i$,
\begin{align*}
|a_i | ~\text{or}~ |b_i| &\le | {\textit{\textbf{X}}_i^{*T}\tilde \bfbeta  - \textit{\textbf{X}}_i^{*T}{\bfbeta ^*}} | + | {\textit{\textbf{X}}_i^{*T}{\bfbeta ^*}}|\\
 &\le | {\textit{\textbf{X}}_i^{*T}\hat \bfbeta  - \textit{\textbf{X}}_i^{*T}{\bfbeta ^*}} | + | {\textit{\textbf{X}}_i^T{\bfbeta ^*}} | \le L(M+B).
\end{align*}

Next, we pose some weighted correlation conditions (WCC):

\textbf{Weighted Correlation Condition (1)}:
\[\mathop {\sup }\limits_{k,j \in H,\atop|{a_i}| \le L(M + B)} \frac{1}{n}\left( {|\sum\limits_{i = 1}^n {{X_{ij}}{X_{ik}}\frac{{\theta^2 {e^{{a_i}}}}}{{{{(\theta  + {e^{{a_i}}})}^2}}}} | \vee |\sum\limits_{i = 1}^n \theta {{X_{ij}}{X_{ik}}(1 - \frac{{\theta {e^{{a_i}}}}}{{{{(\theta  + {e^{{a_i}}})}^2}}})} |} \right) \le \frac{{h{L_1}}}{{ d_H^*}}.\]
\textbf{Weighted Correlation Condition (2)} holds with a high probability:
\begin{align*}
 P\left(\mathop {\sup }\limits_{k,j \in H,\atop|{b_i}| \le L(M + B)}|\frac{1}{n}\sum\limits_{i = 1}^n {\frac{{{X_{ik}}{X_{ij}}{Y_i} \cdot \theta^2{e^{{b_i}}}}}{{{{(\theta  + {e^{{b_i}}})}^2}}}} | \le \frac{h{L_2}}{{ {d_H^*}}}\right)  = 1 - \varepsilon _{n,p},
\end{align*}
where $\varepsilon _{n,p}$ is a constant satisfying $\mathop {\lim }\limits_{n,p \to \infty } \varepsilon _{n,p}^{} = 0$.

By (H.1) and (H.2), $a_i, b_i$ are uniformly bounded random variables, and are viewed as ignorable constants in an asymptotic analysis, as are $\frac{{\theta {e^{{a_i}}}}}{{{{(\theta  + {e^{{a_i}}})}^2}}}$ and $(1 - \frac{{\theta {e^{{a_i}}}}}{{{{(\theta  + {e^{{a_i}}})}^2}}})$. We can check {WCC(2)} using a similar approach to that if the concentration phenomenon for the suprema of the multiplier empirical processes. The conditions above can be obtained by taking a linear transformation of the covariates, that is, by scaling the covariates. WCC(1) is a technical condition used by \cite{Bunea08} for the case of a logistic regression.  This assumption means that the maximum weighted-correlation version of ${\rho _{kl}}~(k \ne l)$ is less than $\frac{{h{L_1}}}{{\theta d_H^*}}$. However, the NBR is more complex than a logistic regression since its Hessian matrix depends on random responses; thus WCC(2) should be assumed with a high probability.

We now have the following constant-free weakest signal detection threshold for correct subset selection.
\begin{theorem}\label{tm:css}
If the assumptions in Theorem~\ref{theo-nbl} hold with ${\varepsilon _n}=0$, under the identifiable condition, WCC(1,2) with $h \le \frac{{a + 2{\lambda _2}}}{{20.25{L_i} + 8a}} \wedge \frac{1}{8}$, for $i=1,2$. Let $p = \exp \{ \frac{1}{{{\rm{1}} - {A^2}}}\log ({2^{{A^2} - 1}}\delta )\} $,
\begin{center}
$P(H = \hat H) \ge 1 -2(1 + d_H^*/p)\delta  - 2p{e^{ - n\lambda _1^2/32C_{LB}^2{L^2}}}- \varepsilon _{n,p},$
\end{center}
provided that the minimal signal condition $\mathop {\min }\limits_{j \in H} |\beta _j^*| \ge 2{\lambda _1}$ is satisfied.
\end{theorem}
\subsection{De-biased elastic-net and confidence interval}
\label{de-biased}
Introduced by \cite{Zhangch14}, the de-biased Lasso was further studied in \cite{Geer2014} and \cite{Jankova2016} within some generalized linear models. Following the the de-biasing idea, we deal with the de-biased estimator $\boldsymbol{\hat b} =: \hat \bfbeta  - \hat \Theta \dot \ell(\hat \bfbeta )$, which is asymptotically normal, based on the established oracle inequality in Section~\ref{High}. Let ${\hat \bfbeta }$ be defined as in optimization problem (\ref{eq:en}). Let ${\hat \Theta }$ be an approximated estimator of the inverse of the Hessian $-\ddot \ell{({\bfbeta ^*})}$(e.g., the CLIME or nodewise Lasso estimator for the estimated Hessian matrix). If $\dot \ell(\hat \bfbeta )$ is continuously differentiable, by Taylor's expansion of vector-valued functions, we have
\begin{align*}
\dot \ell({\bfbeta ^*}) &= \dot \ell(\hat \bfbeta ) - \ddot \ell({\bfbeta ^*})(\hat \bfbeta  - {\bfbeta ^*}) - r({\| {\hat \bfbeta  - {\bfbeta ^{\rm{*}}}}\|_2}) \\
&= \ddot \ell({\bfbeta ^*})[{\bfbeta ^*} - \hat \bfbeta  - \ddot \ell{({\bfbeta ^*})^{ - 1}}\dot \ell(\hat \bfbeta )] - r({\| {\hat \bfbeta  - {\bfbeta ^{\rm{*}}}}\|_2})\\
& = \ddot \ell({\bfbeta ^*})[{\bfbeta ^*} - \hat \bfbeta  + \hat \Theta \dot \ell(\hat \bfbeta )] - \ddot \ell({\bfbeta ^*})[\ddot \ell{({\bfbeta ^*})^{ - 1}} + \hat \Theta ]\dot \ell(\hat \bfbeta ) - r({\| {\hat \bfbeta  - {\bfbeta ^{\rm{*}}}}\|_2})\\
& = :\ddot \ell({\bfbeta ^*})[{\bfbeta ^*} - \hat \bfbeta  + \hat \Theta \dot \ell(\hat \bfbeta )] + {R_n},
\end{align*}
where $r({\| {\hat \bfbeta  - {\bfbeta ^{\rm{*}}}}\|_2}) = {o_p}({\| {\hat \bfbeta  - {\bfbeta ^*}} \|_2})$ is a vector-valued function.

Include $\sqrt n \hat \Theta$ in the equation above if $\sqrt n {R_n}={o_p}(1)$. Then
\[\sqrt n (\boldsymbol{\hat b} - {\bfbeta ^*}) \approx\hat{\mathbf\Theta}[\sqrt n {R_n} - \sqrt n \dot \ell({\bfbeta ^*})] \xrightarrow{d} N(0,\hat{\mathbf\Theta} \mathbf\Sigma {\hat{\mathbf\Theta}^T})\]
where the notation $ \approx $ means asymptotic equivalence under some regular conditions. Here, $\mathbf\Sigma$ is the asymptotic variance of $\sqrt n \dot \ell({\bm\beta ^*})$, where ${\rm{Var}}\dot \ell({\bfbeta ^*}) = \frac{1}{{{n}}}\sum\limits_{i = 1}^n {\frac{{\theta {e^{\bm X_i^T{\bm\beta ^*}}}}}{{\theta  + {e^{\bm X_i^T{\bm\beta ^*}}}}}\bm X_i} {\bm X_i^T}$. We can subsitute in a consistent estimator for $\mathbf\Sigma$ in the high-dimensional case.

The asymptotic confidence level of $1 - \alpha$ for $\beta^*_j$ is then given by
\begin{eqnarray*}
 \bigl[\hat{b}_j - c(\alpha,n,\sigma),
\hat{b}_j + c(\alpha,n,\sigma)\bigr],~~
 c(\alpha,n,\sigma): = \Phi^{-1}(1 - \alpha/2)
\sqrt{(\hat{\mathbf\Theta} \hat{\mathbf\Sigma} \hat{
\mathbf\Theta}^T)_{j,j}/n},
\end{eqnarray*}
where $\Phi(\cdot)$ denotes the c.d.f. of ${N}(0,1)$.

By the KKT conditions in Lemma~\ref{lem:iff}, the de-biased elastic-net estimator is expressed as
$$\boldsymbol{\hat b} =\hat \bfbeta  - \hat \Theta \dot \ell(\hat \bfbeta )=\hat \bfbeta ({\rm I}_p - 2{\lambda _2} \hat \Theta) - \hat \Theta {\lambda _1}{\rm{sign(}}\hat \bfbeta ).$$
A theoretical analysis of the de-biased elastic-net estimator (includeing precision matrix estimation,  confidence interval, and hypothesis testing) is beyond the scope of the this study, please refer to the proofs in \cite{Jankova2016} for some additional details. A simulation study for the de-biased elastic-net is presented in Appendix S4, showing that the de-biased elastic-net has less bias than that of the de-biased Lasso. In the simulation, it is important to estimate the nuisance parameter $\theta$ and estimate the inverse of the Hessian.

\section{Conclusions}

We study sparse high-dimensional NBR problems using several consistency results, such as prediction or $\ell_{q}$-estimation error bounds. NBRs are widely used in modeling count data.  We show that under a few conditions, the elastic-net estimator has oracle properties, which means that when the sample size is large enough, our sparse estimator is very close to the true parameter if the tuning parameters are properly chosen. We also show the sign consistency property under the beta-min condition. We discuss the detection of weak signals, and give a constant-free weakest signal threshold for correct subset selection under some correlation conditions of the covariates. The asymptotic normality of the de-biased elastic-net estimator is also discussed, although doing so further is beyond the scope of this study. These results provide a theoretical understanding of the proposed sparse estimator and provide practical guidance for the use of the elastic-net estimator.

Note that the oracles inequalities in Sections~\ref{Stabil} and \ref{sec-AP} can be extended to many $\ell_1$ or $\ell_1 + \ell_2$ regularized M-estimation regressions with the corresponding empirical process (\ref{eq:EP}) having stochastic Lipschitz properties as presented in Proposition~\ref{prop:upbound1}. For example, the analysis of the stochastic Lipschitz properties of the average negative log-likelihood empirical process can be employed to elastic-net or Lasso penalized COM-Poisson regressions (see \cite{Sellers2008}). As shown in the simulation, the two-step estimation of $\hat \theta$ is not well behave. Like the misspecified models in Example 5.25 of \cite{Vaart1998}, $\theta$, which is a nuisance parameter, is not an important estimate in the consistency results. It would be interesting and important to find a better estimator of $\theta$ in the further research, because $\theta$ is a crucial quantization when constructing confidence interval.

%%%%%%%%%%%%%%%%%%%%%%%%%%%%%%%%%%%%%%%%%%%%%%%%%%%%%%%%%%%%%%%%%%%%%%%%%%%%%%%%%%%%%%%%%%%%%%%%%%%%%%%%%%%%%%%%%%%%%%%%%%%%
% \vskip 14pt
% \noindent {\large\bf Supplementary Materials}
\section*{Supplementary Material}
All proofs and simulation results are in the Supplementary Material.
\par
%%%%%%%%%%%%%%%%%%%%%%%%%%%%%%%%%%%%%%%%%%%%%%%%%%%%%%%%%%%%%%%%%%%%%%%%%%%%%%%%%%%%%%%%%%%%%%%%%%%%%%%%%%%%%%%%%%%%%%%%%%%%
% \vskip 14pt
% \noindent {\large\bf Acknowledgements}
\section*{Acknowledgements}
We are grateful for the kind assistance of Xiaoxu Wu. The authors would like to thank the anonymous referees for their valuable comments. The authors also thank Prof. Cun-Hui Zhang, Prof. Fang Yao and Dr. Sheng Fu for their helpful discussions. This work was partially supported by the National Science Foundation of China (11571021).

\par

%%%%%%%%%%%%%%%%%%%%%%%%%%%%%%%%%%%%%%%%%%%%%%%%%%%%%%%%%%%%%%%%%%%%%%%%%%%%%%%%%%%%%%%%%%%%%%%%%%%%%%%%%%%%%%%%%%%%%%%%%%%%
% \markboth{\hfill{\footnotesize\rm FIRSTNAME1 LASTNAME1 AND FIRSTNAME2 LASTNAME2} \hfill}
% {\hfill {\footnotesize\rm OPTIMAL SUBSAMPLING ALGORITHMS} \hfill}

%\iffalse
\bibhang=1.7pc
\bibsep=2pt
\fontsize{9}{14pt plus.8pt minus .6pt}\selectfont
\renewcommand\bibname{\large \bf References}
%\begin{thebibliography}{11}
\expandafter\ifx\csname
natexlab\endcsname\relax\def\natexlab#1{#1}\fi
\expandafter\ifx\csname url\endcsname\relax
  \def\url#1{\texttt{#1}}\fi
\expandafter\ifx\csname urlprefix\endcsname\relax\def\urlprefix{URL}\fi
%\fi
%\bibliographystyle{chicago}
%\bibliography{report}

% \lhead[\footnotesize\thepage\fancyplain{}\leftmark]{}\rhead[]{\fancyplain{}\rightmark\footnotesize{} }%Put this line in Page 2
% %%%%%%%%%%%%%%%%%%%%%%%%%%%%%%%%%%%%%%%%%%%%%%%%%%%%%%%%%%

%%%%%%%%%%%%%%%%%%%%%%%%%%%%%%%%%%%%%%%%%%%%%%%%%%%%%%%%%%%%%%%%%%%%%%%%%%%%%%%%%%%%%%%%%%%%%%%%%%%%%%%%%%%%%%%%%%%%%%%%%%%%
\vskip .65cm
\noindent
School of Mathematical Sciences and Center for Statistical Science, Peking University, Beijing, 100871, P.R. China
\vskip 2pt
\noindent
E-mail: (zhanghuiming@pku.edu.cn)
\noindent

\vskip 2pt
\noindent Present address: Department of Mathematics, Faculty of Science and Technology, University of Macau, Taipa, Macau,
P.R. China

\noindent
\vskip .65cm
\noindent
School of Public Health and Center for Statistical Science, Peking University, Beijing, 100871, P.R. China
\vskip 2pt
\noindent
E-mail: (jzjia@pku.edu.cn)
\vskip 2pt
% \vskip .3cm
%\centerline{(Received ???? 20??; accepted ???? 20??)}\par
\end{document}